%% file: Foundations_of_Coupled_Nonlinear_Dimensionality_Reduction.tex
\documentclass[twoside]{article}
\usepackage[accepted]{aistats2016}
\usepackage[square,comma,authoryear]{natbib}

\usepackage{amssymb}
\usepackage{amsmath}
\usepackage{times}
\usepackage[mathscr]{euscript}
\usepackage{tikz}
\usepackage{tikz-3dplot}
\usepackage{caption}
\usepackage{mathtools}
\usetikzlibrary{decorations.markings}
\usetikzlibrary{shapes.geometric, arrows}
\usetikzlibrary{3d}
\usepackage{url}

\def\Rset{\mathbb{R}}

\def\Hset{\mathbb{H}}

\DeclareMathOperator*{\E}{\rm E}

\DeclareMathOperator{\rank}{rank}

\providecommand{\iprod}[2]{\left\langle #1 , #2 \right\rangle}
\newcommand{\mat}[1]{\mathbf{#1}}
\newcommand{\tmtextbf}[1]{{#1}}
\newcommand{\bu}{\mat{u}}
\newcommand{\bv}{\mat{v}}
\newcommand{\bz}{\mat{z}}
\newcommand{\bw}{w}
\newcommand{\bPhi}{{\Phi}}
\newcommand{\Mu}{{\boldsymbol \mu}}

\newcommand{\ssigma}{{\boldsymbol \sigma}}

\newcommand{\gapstar}{\Delta_r}

\newcommand{\wt}{\widetilde}
\newcommand{\h}{\widehat}
\newcommand{\ov}{\overline}

\newcommand{\R}{\mathfrak{R}}
\newcommand{\cD}{{\mathcal D}}

\newcommand{\cX}{{\mathcal X}}

\newcommand{\ovK}{\overline{\mat{K}}}
\newcommand{\bK}{\mat{K}}
\newcommand{\eigengap}{\gapstar}

\newcommand{\muset}{{\mathcal M}}
\newcommand{\nuset}{{\mathcal N}}

\newcommand{\supmu}{\sup_{\substack{\Mu \in \muset}}}
\newcommand{\supmuprime}{\sup_{\substack{\Mu \in \nuset}}}
\newcommand{\supmulone}{\sup_{\substack{\| \Mu \|_{1} \leq 1}}}

\newcommand{\msubdelta}{\kappa}
\newcommand{\kyfan}{{\Lambda_{(r)}}}
\newcommand{\oneovermu}{{\nu}}

\newcommand{\m}{u}
\newcommand{\x}{x'}

\newcommand{\set}[1]{\{#1\}}
\newcommand{\ignore}[1]{}

\newtheorem{theorem}{Theorem}[section]
\newtheorem{lemma}[theorem]{Lemma}
\newtheorem{proposition}[theorem]{Proposition}

\newenvironment{proof}[1][Proof]{\begin{trivlist}
\item[\hskip \labelsep {\bfseries #1}]}{\end{trivlist}}

\newcommand{\qed}{\nobreak \ifvmode \relax \else
      \ifdim\lastskip<1.5em \hskip-\lastskip
      \hskip1.5em plus0em minus0.5em \fi \nobreak
      \vrule height0.75em width0.5em depth0.25em\fi}
      
\begin{document}

%

%

\twocolumn[

\aistatstitle{Foundations of Coupled Nonlinear Dimensionality Reduction}


\aistatsauthor{ Mehryar Mohri \And Afshin Rostamizadeh \And Dmitry Storcheus}

\aistatsaddress{Courant Institute and Google Research \And Google Research \And Google Research }]

\begin{abstract}

  In this paper we introduce and analyze the learning scenario of \emph{coupled nonlinear dimensionality reduction}, which combines two major steps of machine learning pipeline: projection onto a manifold and subsequent supervised learning. First, we present new
  generalization bounds for this scenario and, second, we introduce an algorithm that follows from these bounds. The generalization error bound is based on a careful analysis
  of the empirical Rademacher complexity of the relevant hypothesis
  set. In particular, we show an upper bound on the Rademacher
  complexity that is in $\wt O(\sqrt{\kyfan/m})$, where $m$ is the
  sample size and $\kyfan$ the upper bound on the Ky-Fan $r$-norm of
  the associated kernel matrix. We give both upper and lower
  bound guarantees in terms of that Ky-Fan $r$-norm, which strongly
  justifies the definition of our hypothesis set. To the best of our
  knowledge, these are the first learning guarantees for the problem
  of coupled dimensionality reduction. Our analysis and learning
  guarantees further apply to several special cases, such as that
  of using a fixed kernel with supervised dimensionality reduction or
  that of unsupervised learning of a kernel for dimensionality
  reduction followed by a supervised learning algorithm. Based on theoretical analysis, we suggest a structural risk minimization algorithm consisting of the coupled fitting of a low dimensional manifold and a separation function on that manifold. 

\end{abstract}

\section{Introduction}

Classic methods of linear dimensionality reduction assume that data approximately follows some low-dimensional linear subspace and aim at finding an optimal projection onto that subspace, i.e. Principle Component
Analysis (PCA) \citep{pearson1901} and Random Projection \citep{hegde2008random}. Nonlinear dimensionality reduction, also referred to as  \emph{manifold learning}, is a generalization of those linear techniques that aims at fitting a nonlinear low dimensional structure. Such manifold learning methods as Isometric Feature Mapping
\citep{tenenbaum2000global}, Locally Linear Embedding \citep{roweis2000nonlinear}, and Laplacian Eigenmap \citep{belkin2001laplacian} are widely used as methods of nonlinear dimensionality reduction in machine learning,
either to reduce the computational cost of working in
higher-dimensional spaces, or to learn or approximate a manifold more
favourable to subsequent learning tasks such as classification or
regression. These algorithms seek to determine a nonlinear lower dimensional
space by preserving various geometric properties of the input. However,
it is not clear which of these properties would be more beneficial to
the later discrimination stage. Since they are typically unsupervised
techniques, they present a certain risk for the later classification
or regression task: the lower-dimensional space found may not be the
most helpful one for the second supervised learning stage and, in
fact, in some cases could be harmful. How should we design
manifold construction techniques to benefit most the subsequent
supervised learning stage?

As shown by  Figure \ref{fig:example}, simply optimizing geometric
properties may be detrimental the subsequent learning stage. To
solve this problem, we consider an alternative scenario where the
manifold construction step is not carried out \emph{blindly}. We couple the task
of nonlinear dimensionality reduction with the subsequent supervised learning
stage. To do so, we make use of the known remarkable result that all
of the manifold learning techniques already mentioned and many
others are specific instances of the generic Kernel PCA (KPCA)
algorithm for different choices of the kernel function
\citep{ham2004kernel}.  More generally, all these methods can be
thought of first mapping input vectors into a reproducing kernel
Hilbert space and then conducting a low-rank projection within that
space. Thus, our goal is to both learn a mapping as well as a
projection taken from a parametric family as well as a hypothesis
which is found in the low-dimensional space.

The main purpose of this paper is precisely to derive learning
guarantees for this scenario, which we coin as \emph{Coupled Nonlinear Dimensionality Reduction}, and to use those guaranteed as guidelines in the design
of algorithms.

In practice, a user will often use a handful of different kernel
functions and choose the one that is most effective according to
measurements on a validation dataset. Instead, in this work, we argue
that a more effective method is to allow a learning algorithm itself
to choose a kernel function from a parametrized class.  The idea of
automatically selecting a kernel function has been explored in context
of learning algorithms such as Support Vector Machines (SVM)
\citep{lanckriet2004} and Kernel Ridge Regression (KRR)
\citep{cortes2009} (see \citep{gonen2011} and references therein for a
more complete survey). To define the feature mapping,
we will consider kernel families that consist of linear combinations
of fixed \emph{base kernel functions}. Such linear families have been
analyzed extensively in the literature \citep{cortes2010,kloft2011},
however, mainly in the context of kernelized learning algorithms
rather than dimensionality reduction techniques. 
Similarly, to define the projection, we will make use of the top-$r$
eigenspace of a covariance operator that is defined as the linear
combination of the covariances operators of the weighted base kernels.
While some recent
work has considered kernel learning in the setting of dimensionality
reduction \citep{lin2011multiple}, to the best of our knowledge there
has been no theoretical analysis or theoretical justification for the
proposed algorithms. In this work, we provide the necessary
theoretical analysis.

As mentioned above, within the setting of
machine learning, dimensionality reduction is primarily used as a
pre-processing step before regression/classification. For example, the recent work
\citep{dhillon2013risk} illustrates the benefit of dimensionality reduction as preprocessing step by comparing the risk of OLS
regression on reduced data to the risk of ridge regression on full
data. They conclude that the risk of PCA-OLS is at most a constant
factor of the risk of ridge regression, but can often be much less,
as shown empirically.

\begin{figure}[t]
\centering
\scalebox{.5}{\input{example.tikz}}
\captionsetup{singlelinecheck=off}
\caption[.]{An example which illustrates that preserving only geometric
properties may be detrimental to learning on reduced data. The
original data in (a) has four points from blue and red classes. The
eigenvectors of the covariance matrix are $\bv_{1} = \left(
\begin{smallmatrix} 1\\ 0 \end{smallmatrix} \right)$ and $\bv_{2} =
\left( \begin{smallmatrix} 0\\ 1 \end{smallmatrix} \right)$.
Performing standard rank 1 PCA will project both blue and red points
onto $\bv_{1}$, thus merging them (as plotted in (b)). Any
classification on the reduced data will necessarily incur a
classification error of at least $1/2$.}
\label{fig:example}
\end{figure}
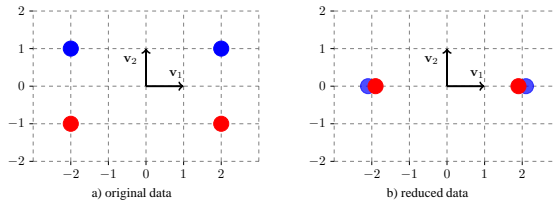

Perhaps unsurprisingly, several empirical investigations have shown
that tuning a dimensionality reduction algorithm in a coupled
fashion, i.e.\ taking into account the learning algorithm that will
use the reduced features, results in considerably better performance
on the learning task \citep{fukumizu2004dimensionality},
\citep{gonen2014coupled}.  Despite this, the vast majority of existing
theoretical analyses of dimensionality reduction techniques (even with
fixed kernel functions) do not directly consider the learning
algorithm that it will be used in conjunction with, and instead focus
on the optimization of surrogate metrics such as maximizing the
variance of the projected features \citep{zwald2006convergence}. One
exception is the work of \citep{mosci2007}, which provides a
generalization guarantee for hypotheses generated by first conducting
KPCA with a fixed kernel and then coupling with a regression model
that minimizes squared loss. There is also a recent work of
\citep{gottlieb2013adaptive}, which derives generalization bounds based
on Rademacher complexity for learning Lipschitz functions in a general
metric space. They show that the intrinsic dimension of data
significantly influences learning guarantees by bounding the
corresponding Rademacher complexity in terms of dimension of
underlying manifold and the distortion of training set relative to
that manifold.

In our setting, we consider hypotheses which include both the KPCA
dimensionality reduction step, with a \emph{learned} linear combination
kernel, as well as a linear model which uses the reduced features for
a supervised learning task.

Although the hypothesis set we analyze is most naturally associated to
a ``Coupled Nonlinear Dimensionality Reduction'' algorithm, which jointly selects both a kernel
for nonlinear projection as well as a linear parameter vector for
a supervised learning task, we note that this hypothesis set also
encompasses algorithms that proceed in two stages, i.e.\ by first
selecting a manifold and then learning a
linear model on it.

The results of this paper are organized as follows: in the following
section we outline the learning scenario, including the hypothesis
class, regularization constraints as well as define notation.
Section~\ref{sec:generalization} contains our main result, which is an
upper bound on the sample Rademacher complexity of the proposed
hypothesis class that also implies an upper bound on the
generalization ability of the hypothesis class. In
Section~\ref{sec:lower bound} we show a lower bound on the sample
Rademacher complexity as well as other quantities, which demonstrates
a necessary dependence on several crucial quantities and helps to
validate the design of the suggested hypothesis class. In
Section~\ref{sec:discussion} we provide a short discussion of the
implications of our theoretical results, which leads us to Section~\ref{sec:algorithm}, where develop an algorithm for the coupled fitting of a kernel and a separation function.

\section{Learning scenario}
\label{sec:Learning scenario}

\begin{figure}[t]
\centering
\scalebox{.7}{\input{scenario.tikz}}
\caption{Flow chart illustrating the Coupled Dimensionality Reduction learning scenario.}
\label{fig:learning-scenario}
\end{figure}
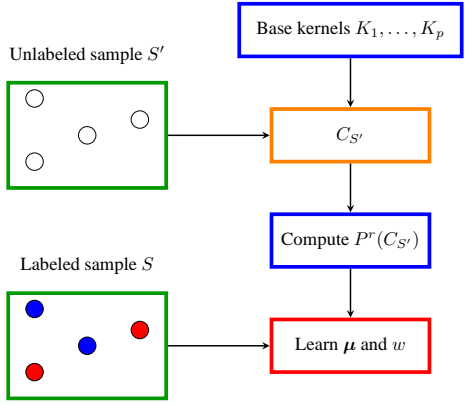

Here, we describe the learning scenario of supervised dimensionality
reduction. Let $\cX$ denote the input space.  We assume that the
learner receives a labeled sample of size $m$, $S = ((x_1, y_1),
\ldots, (x_m, y_m))$, drawn i.i.d.\ according to some distribution
$\cD$ over $\cX \times \set{-1, +1}$, as well as an unlabeled sample
$U = (\x_1, \ldots, \x_{\m})$ of size $\m$, typically with $\m \gg m$,
drawn i.i.d.\ according to the marginal distribution $\cD_\cX$ over
$\cX$.

We assume that the learner has access to $p$ positive-definite
symmetric (PDS) kernels $K_1, \ldots, K_{p}$. Instead of requiring the
learner to commit to a specific kernel $K$ defining KPCA with a
solution subsequently used by a classification algorithm, we consider
a case where the learner can define a dimensionality reduction
solution defined based on $\sum_{k = 1} \mu_k K_k$, where the
non-negative mixture weights $\mu_k$ are chosen to minimize the error
of the classifier using the result of the dimensionality reduction.

Given $p$ positive-definite symmetric (PDS) kernels $K_1, \ldots
,K_{p}$ and a vector $\Mu \in \Rset^{p}$ with non-negative
coordinates, consider a set of weighted kernel functions
$\{\mu_1K_1,...,\mu_{p}K_{p}\}$, where each $\mu_kK_k$ has its
reproducing space $\Hset_{\mu_kK_k}$.  The unlabeled sample $U$
is used to define the empirical covariance operator of each weighted
base kernel $\mu_kK_k$, denoted as $C_{U,\mu_kK_k}\colon
\Hset_{\mu_kK_k} \to \Hset_{\mu_kK_k}$.  Given a set of covariance
operators $\{C_{U,\mu_kK_k}\}_{k=1}^{p}$ we define an operator
$C_{U,\Mu} = \sum_{k=1}^{p}C_{U,\mu_kK_k}$ that acts on the
sum of reproducing spaces
$\Hset_{\Mu}=\Hset_{\mu_1K_1}+...+\Hset_{\mu_{p}K_{p}}$.

Let $P_{r}(C_{U,\Mu})$ denote the rank-$r$ projection over the
eigenspace of $C_{U,\Mu}$ that corresponds to the top-$r$
eigenvalues of $C_{U,\Mu}$ denoted as $\lambda_1(C_{U,\Mu})
\geq \ldots \geq \lambda_r(C_{U,\Mu})$.\footnote{
Although it is also possible provide results in an even more
general setting, we make the assumption that the chosen $r$ satisfies
$\lambda_{r}(C_{U,\Mu}) \neq \lambda_{r+1}(C_{U,\Mu})$ in
order to simplify the presentation. Note, assumption this is always
satisfied if the eigenvalues are simple.
}
Let $\bPhi_{\mu_kK_k} \colon
\cX \to \Hset_{\mu_kK_k}$ denote the feature mapping associated to
$\mu_k K_k$, specifically for each $x\in\cX$, we have the function
$\bPhi_{\mu_kK_k}(x) = \sqrt{\mu_k} K_k(x, \cdot)$. Define
$\bPhi_{\Mu}\colon \cX \to \Hset_{\Mu}$ as
$\bPhi_{\Mu} = \sum_{k=1}^{p}\bPhi_{\mu_kK_k}$. This parametrized
projection is used to define rank-$r$ feature vectors (functions)
$P_{r}(C_{U,\Mu}) \bPhi_{\Mu}(x)$.

From this point onward, in to avoid intricate notation, we will not
not explicitly indicate the dependence of $\Hset_\Mu$, $C_{U,
  \Mu}$ and $\bPhi_\mu$ on $\Mu$ and instead use $\Hset$, $C_{U}$
and $\bPhi$.  Similarly, we will refer to $C_{U,\mu_kK_k}$ as
$C_{U,k}$ and $\bPhi_{\mu_kK_k}$ as $\bPhi_k$. We will also use
the shorthand $\Pi_{U}$ (resp. $\Pi_S$) instead of
$P_{r}(C_{U})$ (resp.  $P_{r}(C_S)$).

Once a projection is defined, the labeled sample $S$ is used to learn
a linear hypothesis $x \mapsto \iprod{\bw}{\Pi_{U}
\bPhi(x)}_\Hset$ with bounded norm, $\| \bw \|_\Hset \leq 1$, in the
subspace of the projected features $\Pi_{U} \bPhi(x)$.

The two steps just described, dimensionality reduction by projection $\Pi_{U}$ and supervised learning of $\bw$,
may be coupled so that the best choice of weights $\Mu$ is
made for the subsequent learning of $\bw$.
(see Figure~\ref{fig:coupled}).

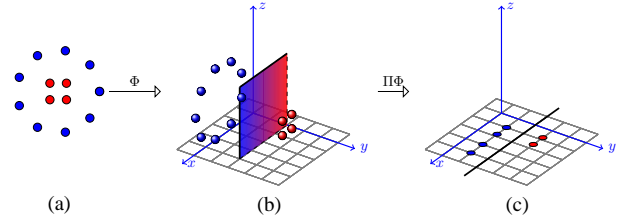
\begin{figure}[t]
\centering
\scalebox{.55}{\input{coupled.tikz}}
\caption{Illustration of the coupled learning problem.  (a) Raw input
  points. (b) Points mapped to a higher-dimensional space where linear
  separation is possible but where not all dimensions are
  relevant. (c) Projection over a lower-dimensional space preserving
  linear separability.}
\label{fig:coupled}
\end{figure}

To do so, given constants
$\kyfan$ and $\oneovermu$, we select vector $\Mu$ out of a convex set $\muset=$
\begin{equation*}
\label{eq:mu-domain}
  \left\{
     \Mu\colon
     \|\Mu\|_{(r)} \leq \kyfan, ~~
     \|\Mu\|_1 \leq 1, ~~
     \sum_{k=1}^{p} \frac{1}{\mu_k} \leq \oneovermu, ~~
     \Mu \geq 0 \right\} \,.
\end{equation*}
We will show that the choice of this convex regularization set is
crucial in guaranteeing the generalization ability our hypothesis
class.  The vector $\Mu$ is upper bounded by an $L_1$-norm inequality $\|
\Mu \|_1 \leq 1$ (standard from learning kernels literature) but
also by an inequality $\| \Mu \|_{(r)} \leq \kyfan$, where $\| \cdot
\|_{(r)}$ is the semi-norm defined as the Ky Fan $r$-norm of $C_{U}$
\citep{Bhatia1997}:
\begin{equation}
\label{eq:semi-norm}
\| \Mu \|_{(r)} = \| C_{U}\|_{(r)} = \sum_{i = 1}^r \lambda_i(C_{U}) \,.
\end{equation}
The use of the semi-norm $\| \cdot \|_{(r)}$ in this
context is key since $\| \Mu \|_{(r)} $ appears as the relevant
quantity both in our generalization bounds and in our lower bounds.
The lower bound constraint on $\Mu$, $\sum_{k=1}^{p} \mu_k^{-1} \leq
\oneovermu$, will imply an upper bound on the eigengap of the induced
covariance operator, which is a fundamental quantity that influences
the concentration of eigenspaces. In fact in Section~\ref{sec:lower
bound} we give a simple example that demonstrates the dependence on
the eigengap is tight, and also implies the
necessity of the lower bound regularization.

Thus, the hypothesis set $H$ defined by our supervised dimensionality reduction
set-up is defined as follows:
\begin{equation}
\label{eq:hypo-set}
H = \bigg\{ x \mapsto 
\iprod{\bw}{\Pi_{U} \bPhi(x)}_{\Hset} \colon 
\|\bw\|_{\Hset} \leq 1, \Mu \in \muset \bigg\} \,.
\end{equation}
In the analysis that follows, we will
also make use of normalized kernel matrices which, given a sample $S$
of size $m$ and kernel function $K$, are defined as $[\ovK]_{i,j} =
\frac{1}{m} K(x_i, x_j)$. Since the kernel matrix is
normalized we have that $\lambda_i(\ovK_k) = \lambda_i(C_{S,k})$, where
$C_{S,k}$ is the sample covariance operator associated to $K_k$ (see
\citep{rosasco2010learning} Proposition~9.2). We also define the unscaled
sample kernel matrix $[\mat{K}]_{i,j}=K(x_{i},x_{j})$.

We will assume that $C_{U}$ admits at least $r$ non-zero eigenvalues
and will similarly assume that the set of kernel matrices
$\ovK_k$ of $K_k$ associated to the sample $U$ or $S$ for any $k
\in [1, p]$ contains at least one matrix with rank at least $r$.
Furthermore, we assume that the base kernels $K_k$, $k \in [1, p]$, satisfy
the condition $K_k(x, x) \leq 1$ for all $x \in \cX$, which is
guaranteed to hold for all normalized kernels. Finally, we assume the
base kernels are \emph{linearly independent with respect to the union
of the samples $S$ and $U$}.

\paragraph{Definition 1. Linearly Independent Kernels}
\label{def:lin-indep-kernels}
Let $K_1, \ldots, K_p$ be $p$ PDS kernels and let $S = (x_1, \ldots,
x_n)$ be a sample of size $m$. For any $k \in [1, p]$, let $\Hset_k$
denote the RKHS associated to $K_k$ and $\ov \Hset_k$ the subspace of
$\Hset_k$ spanned by the set of functions $\set{\bPhi_{K_k}(x_i)\colon
i = 1, \ldots, m}$. Then, $K_1, \ldots, K_p$ are said to be
\emph{linearly independent with respect to the sample $S$} if, for any
$k \in [1, p]$, no non-zero function in $\ov \Hset_k$ can be expressed
as a linear combination of the functions in $\cup_{l \neq k} \ov \Hset_l$.

Linear independence typically holds in practice, e.g., for polynomial
and Gaussian kernels on $\Rset^N$. 
As an example, let $\cX=\Rset^N$ and define the sample
$S=\{x_1,\ldots,x_m\}$.
Define two base kernels: Gaussian $K_1(x,y)=e^{-\|x-y\|^2}$ and
linear $K_{2}(x,y)=\iprod{x}{y}$. Then $\bPhi_{K_1}(x) \colon t \mapsto
e^{-\|x-t\|^{2}}$, i.e.\ $\bPhi_{K_1}(x)$ is an exponential
function $e^{-\|x-t\|^{2}}$ with parameter $x$ and argument $t$. In the
same manner $\bPhi_{K_{2}}(x) \colon t \mapsto \iprod{x}{t}$. Thus,
$\ov\Hset_1$ is the span of exponential functions
$\{e^{-\|x_1-t\|^{2}}, \ldots, e^{-\|x_m-t\|^{2}}\}$ and $\ov\Hset_{2}$ is
the span of linear functions $\{\iprod{x_1}{t}, \ldots,
\iprod{x_m}{t}\}$. Clearly, no exponential function can be represented
as a linear combination of linear functions and likewise, in general,
no linear function is represented as a (finite) linear combination of
exponential functions. Thus, the base kernels $K_1$ and $K_{2}$
are linearly independent with respect to sample $S$ as in Definition~\ref{def:lin-indep-kernels}. Therefore $\ov\Hset_1\perp
\ov\Hset_{2}$ in the reproducing space of $K_1+K_{2}$ and
defined $\ov \Hset=\ov\Hset_1\bigoplus \ov\Hset_{2}$. Such
decomposition of $\ov\Hset$ into a direct sum allows to characterize
the eigenfunctions of $C_{S}$: they consist of the eigenfunctions of
$\left. C_{S}\right|_{\ov\Hset_1}$ and $\left.
C_{S}\right|_{\ov\Hset_{2}}$. The eigenfunctions of each restricted operator
are orthogonal to each other. Note that the orthogonality of
eigenfunctions does not necessarily imply the orthogonality of the eigenvectors of sample kernel
matrices $\ovK_1$ and $\ovK_{2}$.

More generally, the support of the
base kernels can be straightforwardly modified to ensure that this
condition is satisfied. 

It follows from construction $\Hset=\Hset_1+ \cdots +\Hset_{p}$ and the results of \cite[Section 6]{aronszajn1950theory} that when base kernels are linearly independent with respect to sample $S$, then $\ov\Hset_k$ are orthogonal subspaces of $\Hset$, thus we can define $\ov\Hset=\bigoplus_{k = 1}^{p} \ov \Hset_k$, which will be
extremely useful in decomposing the spectra of operators $C_{S}$.
Linearly independent base kernels imply that $C_{S}$ has at most $pm$
nonzero eigenvalues of the form $\mu_k\lambda_{j}(C_{S, k})$, which is an explicit representation of the eigenvalues of $C_{S}$ in terms of $\Mu$.

\section{Generalization bound}
\label{sec:generalization}

In this section we outline the main steps taken in deriving a
generalization bound as well as analyze the bound and discuss its
implications. Proofs that are not included in this section can be
found in the appendix.

The main result of this section is to derive an upper bound on the sample
Rademacher complexity $\h \R_S(H)$ of hypothesis class
$H$.  The sample Rademacher complexity of $H$ is defined as $ \h
\R_S(H)  = \frac{1}{m} \E_{\ssigma} \Big[\tmtextbf{\underset{h \in
H}{\sup}} \sum_{n = 1}^m \sigma_n h(x_{n})  \Big]$, where $\sigma_{n}$
are i.i.d.\ random variables taking values +1 and -1 with equal
probabilities.
Once that is done we can then directly invoke the result of
\citep{koltchinskii2002empirical} and \citep{bartlett2003rademacher},
which states that with probability at least $1-\delta$ over the draw
of sample $|S| = m$ and for all $h \in H$ the generalization error
$R(h)$ is bounded by
\begin{equation}
\label{eq:koltchinski}
R(h) \leq \h R_{S,\rho(h)} + \frac{2}{\rho} \h \R_S(H) +
3\sqrt{\frac{\log{(2 / \delta)}}{2m}} \,,
\end{equation}
where, given $c(x)$ is the true label of $x \in \cX$, $R(h)=Pr_{x \sim \cD}[h(x)\neq c(x)]$ and $\hat{R}_{S,\rho}(h)$ is the fraction of points in $S$ with
classification margin less than $\rho$.

Note, in our setting $H$ is parametrized by $\bw$ and $\Mu$ and that
we can consider the supremum over these two parameters separately.
Finding the supremum over $\bw$ can be done in a standard manner, using
Cauchy-Schwarz,
\begin{align*}
  \underset{\| \bw \| \leq 1}{\sup} \sum_{n = 1}^m \sigma_n h(x_{n})
   & = \underset{\| \bw \| \leq 1}{\sup} \langle \bw, \Pi_{U}\sum_{n = 1}^m \sigma_n \bPhi(x_{n})\rangle \\
  & =  \| \Pi_{U}\sum_{n = 1}^m \sigma_n \bPhi(x_{n})\| \,.
\end{align*}
Now it remains to compute $\supmu\| \Pi_{U}\sum_{n = 1}^m
\sigma_n \bPhi(x_{n})\|$, which is the more challenging expression.
First of all, it will be more convenient to work with a
projection defined with respect to $C_S$ instead of $C_{U}$, since we
are projecting instances from sample $S$. Similarly, we
will find it useful to control a norm $\|C_{S}\|_{(r)}$ instead of $\|C_{U}\|_{(r)}$.
Both of these
issues can be addressed by using concentration inequalities to bound
the difference of the projections $\Pi_{U}$ and $\Pi_{S}$
\citep{zwald2006convergence} as well as the difference of the
operators $C_{U}$ and $C_{S}$ \citep{shawe2003estimating}. For that we extend a constraint set $\muset$ to a larger set $\nuset$:
\begin{multline}
 \nuset = \Big\{
   \Mu\colon 
     \| C_{S} \|_{(r)} \leq \kyfan + \msubdelta, ~~
     \|\Mu\|_1 \leq 1, ~~\\
     \sum_{k=1}^{p} \frac{1}{\mu_k} \leq \oneovermu, ~~
     \Mu \geq 0 \Big\} \,,
\end{multline}
 where $\msubdelta=4
\Big(1+\sqrt{\frac{\log{(2p / \delta)}}{2}}\Big)$. In the following lemma we bound the transition from $\supmu\|
\Pi_{U}\sum_{n = 1}^m \sigma_n \bPhi(x_{n})\|$ to $\supmuprime\| \Pi_{S}\sum_{n = 1}^m
\sigma_n \bPhi(x_{n})\|$.
\ignore{
Thus, we will
approximate $\underset{\| \Mu \|_{(r)} \leq \kyfan}{\sup}\|
\Pi_{U}\sum_{n = 1}^m \sigma_n \bPhi(x_{n})\|$ by $\underset{\|
\Mu \|_{\lambda} \leq \kyfan + \epsilon}{\sup}\| \Pi_{S}\sum_{n
= 1}^m \sigma_n \bPhi(x_{n})\|$, where $\epsilon$ is defined below.
Theorem~\ref{th:concentration-KPCA} guarantees that due to the
convergence of eigenspaces of sample covariance operators to that of
population covariance, with high probability $\Pi_{S}$ will be close to
$\Pi_{U}$ and $\| \Mu \|_{\lambda}$ close to $\| \Mu \|_{(r)}$.
}
\begin{lemma}
\label{lemma:approximation of supremum}
Let $C_{S,k}$ be the sample covariance operator of kernel $\mu_kK_k$ with a reproducing space $\Hset_k$. Define $C_{S}$ (resp.
$C_{U}$) as $C_{S}=\sum_{k=1}^{p}C_{S,k}$ and $\Pi_{S}$ (resp.
$\Pi_{U}$) be the orthogonal projection onto the eigenspace of
$\lambda_{i}(C_{S})$ for $i \in [1,r]$.
Then with probability at least $1-\delta$ for any $u \in \Hset=\Hset_1+ \cdots + \Hset_{p}$
\begin{equation}
\sup_{\Mu \in \muset}\| \Pi_{U}u\|
\leq \sup_{ \Mu \in \nuset }  \Big(\|
\Pi_{S}u\| + \frac{8\msubdelta\oneovermu }{\eigengap\sqrt{m}}
\| u \| \Big)\,,
\end{equation}
where $\eigengap = \min_{k\in[1,p]}
\big(\lambda_{r}(C_k)-\lambda_{r+1}(C_k)\big)$, $C_k$ is the true covariance operator of kernel $K_k$ and $\msubdelta=4
\Big(1+\sqrt{\frac{\log{(2p / \delta)}}{2}}\Big)$ \,.
\end{lemma}

Now using Lemma~\ref{lemma:approximation of supremum} and letting
$u=\sum_{n = 1}^m \sigma_n \bPhi(x_{n})$ we find that, with high
probability, the Rademacher complexity of $H$ is bounded by 

\begin{align}
\begin{split}
\label{eq:expectations}
\h \R_S(H) & \leq \frac{1}{m} \E_{\ssigma} \Bigg[ \sup_{\Mu \in \nuset}
  \Big(\| \Pi_{S} \sum_{n = 1}^m \sigma_n \bPhi(x_{n}) \| \\
   & + \frac{8\msubdelta\oneovermu}{\eigengap\sqrt{m}} \| \sum_{n = 1}^m \sigma_n \bPhi(x_{n}) \| \Big)
     \Bigg]  \,.
\end{split}
\end{align}
We will distribute the supremum and bound each of the two terms
separately. In the case of the second term, $\| \sum_{n = 1}^m
\sigma_n \bPhi(x_{n}) \|$,
we will upper bound the supremum over $\Mu \in \nuset$ with the
supremum over a larger set constrained only by $\| \Mu \|_1\leq
1$. This leads us to the expression $ \frac{1}{m} \E_{\ssigma} \big[ \supmulone \| \sum_{n = 1}^m \sigma_n \bPhi(x_{n}) \| \big]$, which is exactly equal to
the Rademacher complexity of learning kernels for classification
without projection. This complexity term can be bounded using
Theorem 2 of \citep{cortes2010}, which gives the following:
\begin{equation}
\frac{1}{m} \E_{\ssigma} \Bigg[ \underset{\| \Mu \|_1 \leq
1}{\sup} \| \sum_{n = 1}^m \sigma_n \bPhi(x_{n}) \| \Bigg] \leq
\frac{\sqrt{ \eta_{0} e \lceil \log{p} \rceil }}{\sqrt{m}} \,,
\end{equation}
where $\eta_{0}=\frac{23}{22}$.
Now it remains to bound the expectation of $\| \Pi_{S} \sum_{n =
1}^m \sigma_n \bPhi(x_{n}) \|$.
\ignore{
It is interesting to note that this
term will approach the term found in the complexity of the learning
kernels
hypothesis class as $r \to m$ and $P_{r}(C_S) \to I$. If $r
< m$, we may in fact hope to find a bound that is tighter than the one found
in the learning kernels literature.
}

\begin{lemma}
\label{lemma:supremum}
Let $C_k$ and $C_{S,k}$ be the true and sample covariance operator
of kernel $K_k$. Define $C_{S}$ (resp.  $C_{U}$) as
$C_{S}=\sum_{k=1}^{p}C_{S,k}$ and let $\Pi_{S}$ (resp.
$\Pi_{U}$) be the orthogonal projection onto the eigenspace of
$\lambda_{i}(C_{S})$ (resp. $\lambda_{i}(C_{S'})$) for $i \in [1,r]$.
Then with probability at least $1-\delta$. 
\begin{multline}
 \frac{1}{m} \E_{\ssigma} \Bigg[ \supmuprime \Big(\| \Pi_{S}
   \sum_{n = 1}^m \sigma_n \bPhi(x_{n}) \|  \Big)
     \Bigg] \leq  \\ 
     \frac{1}{\sqrt{m}}\sqrt{2\big(\kyfan+
     \msubdelta\big)\log{(2 p \textbf{ }m)}} 
     \,,
\end{multline}
where $\msubdelta=4
\Big(1+\sqrt{\frac{\log{(2p / \delta)}}{2}}\Big)$.
\end{lemma}
\begin{proof}
The term $\| \Pi_{S}
   \sum_{n = 1}^m \sigma_n \bPhi(x_{n}) \|$ is naturally
bound using the constraint on $\| C_{S} \|_{(r)}$ since it involves
projection onto eigenspace of $C_{S}$ and $\| C_{S}
\|_{(r)}$ controls its spectrum. Therefore we
will reduce the problem to the supremum over $\Vert C_{S}
\Vert_{(r)} \leq \epsilon$, where $\epsilon=\kyfan + \msubdelta$.

By Lemma~\ref{lemma:inner_product_form} (see appendix) we have that
$\| \Pi_{S} \sum_{n = 1}^m \sigma_n \bPhi(x_{n}) \|^{2}= m \bu_{\mu}
\cdot \bu_{\sigma}$, where $\bu_{\mu}$ is a vector with entries
$\mu_k\lambda_{j}(\ovK_k)$ and $\bu_{\sigma}$ is a vector with
entries $(\bv_{k,j}^{\top}\ssigma)^{2}$ such that $(k,j) \in I_{\mu}$.
An indexing set $I_{\mu}$ is a set of pairs $(k,j)$ that correspond to
largest $r$ eigenvalues of $\{\mu_k\lambda_{j}(\ovK_k)\}_{k,j}$.
In order remove the dependence of the indexing set on the identity of
the top eigenvalues, we upper bound the expression over the choice of
all size-$r$ sets:
\begin{equation}
\label{eq:introducing sum}
\underset{\| C_{S} \|_{(r)} \leq \epsilon}{\sup} \bu_{\mu} \cdot
\bu_{\sigma} 
= \sup_{\| \bu_{\mu} \|_1 \leq \epsilon} \bu_{\mu} \cdot
\bu_{\sigma}
\leq \underset{|I|=r}{\sup} \Big(\underset{\| \bu_{\mu}
\|_1 \leq \epsilon}{\sup} \bu_{\mu} \cdot \bu_{\sigma}\Big) \,,
\end{equation}
where $\sup_{|I|=r}$ indicates the supremum over all indexing sets
$I$ of size $r$. 
\ignore{
We needed an extra supremum since $\bu_{\sigma}$ is
tied to the choice of indexing set. Now when dependence on the
particular indexing set is washed away, the norm $\| \Mu \|_{\lambda}$
is equal to $\| \bu_{\mu}\|_1$ just by construction of vector
$\bu_{\mu}$. Note that inequality ~\eqref{eq:introducing sum} is tight
and the tight case happens when we have only one base kernel.
}
Then, by the dual norm property we have
\begin{equation}
\sup_{|I|=r} \underset{\| \bu_{\mu} \|_1 \leq \epsilon}{\sup}
\bu_{\mu} \cdot \bu_{\sigma}=\underset{|I|=r}{\sup}\epsilon\|
\bu_{\sigma}
\|_{\infty}=\epsilon \max_{k,j}(\bv_{k,j}^{\top}\ssigma)^{2}
\,.
\end{equation} 
Thus, $\| \Pi_{S} \sum_{n = 1}^m \sigma_n \bPhi(x_{n}) \|$ is bounded by the following:

\begin{align*}
 \underset{k,j}{\max}\sqrt{m \epsilon(\bv_{k,j}^{\top}\ssigma)^{2}} 
 & \leq \sqrt{m \epsilon} \underset{k,j}{\max}|\bv_{k,j}^{\top}\ssigma| \\
& =\sqrt{m \epsilon} \underset{k,j}{\max}\underset{s_{t}\in
\{-1,1\}}{\max}s_{t}\bv_{k,j}^{\top}\ssigma \,.
\end{align*}
By Massart's lemma \citep{massart2000some}
\begin{equation}
\E_{\ssigma} \Big[ \underset{k,j}{\max}\underset{s_{t}\in
\{-1,1\}}{\max}s_{t}\bv_{k,j}^{\top}\ssigma
\Big]\leq \sqrt{2 \log{(2 p m)}} \,.
\end{equation}
This follows since
the norm of $s_{t}\bv_{k,j}$ is bounded by
$1$ and the cardinality of the set which the maximum is taken
over is bounded by
$ 2 \sum_{k=1}^{p} \sum_{j=1}^{\rank(\ovK_k)} \leq 2pm$.

Combining all intermediate results brings us to the bound
\begin{multline}
 \frac{1}{m} \E_{\ssigma} \Bigg[ \underset{\| C_{S} \|_{(r) } \leq \epsilon}{\sup} \| \Pi_{S} \sum_{n = 1}^m \sigma_n \bPhi(x_{n}) \|
     \Bigg] \leq \\
      \frac{1}{\sqrt{m}} \sqrt{2 \epsilon \log{(2 p \textbf{
}m)}},
\end{multline}
and the final result is obtained by letting $\epsilon=\kyfan+ \msubdelta$.
\qed
\end{proof}

Thus, after combining Lemma~\ref{lemma:supremum}
and~\ref{lemma:approximation of supremum} above we derive an upper
bound on the expectation in \eqref{eq:expectations}, which gives us a
bound on the sample Rademacher complexity.  That bound is presented in
the following theorem.

\begin{theorem}
\label{theorem:rademacher_complexity}
Let hypothesis set $H$ be defined as in \eqref{eq:hypo-set}. Then for any sample $S$ of size $m<\m$ drawn i.i.d.\
according to some distribution $\cD$ over $\cX \times \set{-1, +1}$ such that $\sqrt{m} >
\frac{2 \msubdelta}{\eigengap}$ the empirical Rademacher complexity
of the hypothesis set $H$ can be bounded as follows with probability
at least $1-\delta$,
\begin{align}
\begin{split}
\label{eq:rademacher_bound}
\h \R_S(H) & \leq 
\frac{1}{\sqrt{m}}\Bigg(\sqrt{2  \big(\kyfan+\msubdelta\big)\log{(2 p m)}} \\
& +\frac{8\msubdelta\oneovermu\sqrt{ \eta_{0} e \lceil \log{p} \rceil }}{\eigengap}\Bigg)\,,
\end{split}
\end{align}
where
$\eigengap=\min_k\big(\lambda_{r}(C_k)-\lambda_{r+1}(C_k)\big)$, $\msubdelta=4
\Big(1+\sqrt{\frac{\log{(2p / \delta)}}{2}}\Big)$ and $\eta_0 =
\frac{23}{22}$.
\end{theorem}

If we consider only parameters $p$ and
$\kyfan$, then the Rademacher complexity bound in~\eqref{eq:rademacher_bound} is $O\Big(\sqrt{\frac{\kyfan\log(p m)}{m}}\Big)$. The learning scenario and regularization in standard learning kernels \citep{cortes2010} differs from ours, thus we will make a few adjustments that will allow us to compare those two bounds in a most coherent way, particularly, we will let $S=U$ and express $\kyfan$ in terms of unscaled sample kernel matrices. 
\begin{equation}
\kyfan=\frac{1}{m}\sup_{|I|=r}\sum_{(k,j)\in I}\mu_k\lambda_{j}(\bK_k) \leq \frac{1}{m}\sup_{|I|=r}\sum_{(k,j)\in I}\lambda_{j}(\bK_k)
\end{equation}
That results in Rademacher complexity of $O\Big(\frac{1}{m}\sqrt{\sup_{|I|=r}\sum_{(k,j)\in I}\lambda_{j}(\bK_k)\log{(pm)}}\Big)$,
  while the standard learning kernels bound is $O\Big(\frac{1}{m}\sqrt{\sup_{k \in [1,p]}\sum_{j=1}^{m}\lambda_{j}(\bK_k)\log{(p)}}\Big)$.
Here, $\sup_{|I|=r}\sum_{(k,j)\in I}\lambda_{j}(\bK_k)$ is the largest $r$-long sum of eigenvalues that can be picked from any base kernel matrix, while $\sup_{k \in [1,p]}\sum_{j=1}^{m}\lambda_{j}(\bK_k)$ is the largest $m-$ long sum of eigenvalues that can be picked only from one base kernel matrix. Thus, if $r$ is sufficiently smaller than $m$ and given a certain choice of base kernels, learning kernels in the supervised dimensionality reduction problem will enjoy a tighter Rademacher complexity than standard learning kernels.

Finally, plugging in the upper bound from
Theorem~\ref{theorem:rademacher_complexity} into
\eqref{eq:koltchinski} results in the generalization bound of
$H$. Note that the confidence term in ~\eqref{eq:koltchinski} changes
from $\log{(2p / \delta)}$ to $\log{(4p / \delta)}$, because
Rademacher complexity is bounded with high probability.  To the best
of our knowledge, this is the first generalization guarantee provided
for the use of projection in the reproducing space with a learned kernel in a supervised learning
setting.

\begin{theorem}
\label{theorem:generalization_bound}
Let hypothesis set $H$ be defined as in \eqref{eq:hypo-set}. Then with probability at least $1-\delta$ over the draw
of sample $|S| = m$ and for all $h \in H$ the generalization error
$R(h)$ is bounded by
\begin{align*}
\label{eq:generalization_bound}
\h R_{S,\rho}(h) & + \frac{2}{\rho\sqrt{m}}\Bigg(\sqrt{2  \big(\kyfan+\msubdelta\big)\log{(2 p m)}} \\
& +\frac{8\msubdelta\oneovermu\sqrt{ \eta_{0} e \lceil \log{p} \rceil }}{\eigengap}\Bigg) +
3\sqrt{\frac{\log{(4p / \delta)}}{2m}}  \,, 
\end{align*}
where $\hat{R}_{S,\rho}(h)$ is the fraction of points in $S$ with
classification margin less than $\rho$.
\end{theorem}
We note that \citep{mosci2007} and \citep{gottlieb2013adaptive} provide
generalization bounds for supervised dimensionality reduction, however
their learning scenarios are different from ours, particularly in a
sense that they do not learn a mapping and projection for
dimensionality reduction jointly with a discrimination function on
reduced data. Nevertheless, their generalization bounds are comparable
to ours in the special case $p=1$. The analysis of  \citep{gottlieb2013adaptive} is
done for general metric spaces, while in the particular example of
Euclidean space they show that generalization bound is
$O\big(\sqrt{\frac{\kappa}{m}}+\sqrt{\frac{\eta}{m}}\Big)$, where
$\kappa$ is the dimension of underlying data manifold and $\eta$ is
the average distance of training set to that manifold. Thus, while our
bound has the same rate with respect to $m$ as that of
\citep{gottlieb2013adaptive}, it is based on Ky-Fan semi-norm
regularization $\kyfan$, while the bound of
\citep{gottlieb2013adaptive} is based on the assumption that data
approximately follows some low dimensional manifold with dimension
$\kappa$. The generalization bound in \citep{mosci2007} is
$O\big(\frac{1}{\sqrt{m}}\Big)$, however while we fix the number of
eigenvalues for dimensionality reduction at $r$, their bound requires
selecting all eigenvalues above a threshold
$\lambda_{m}=O\Big(\frac{1}{\sqrt{m}}\Big)$. By applying this
threshold they have control over the eigenspectrum of covariance
operator, while we control the spectrum by the Ky-Fan regularization.

\section{Lower bounds}
\label{sec:lower bound}

In this section we show a lower bound on the Rademacher complexity of the
hypothesis class $H$ defined in \eqref{eq:hypo-set}. 
This lower bound demonstrates the central role that the Ky Fan
$r$-norm regularization, $\|\mu\|_{(r)} \leq \kyfan$, plays in
controlling the complexity of the hypothesis class. Furthermore, it
demonstrates the tightness of the upper bound presented in the
previous section in terms the number
of training samples $m$. We additionally give a small example that
demonstrates the necessity of the eigengap term which appears in
Lemma~\ref{lemma:approximation of supremum} and which motivates the
additional regularization term $\sum_{k=1}^p \mu_k^{-1} \leq \oneovermu$
that is used to bound the eigengap in the proof of
Lemma~\ref{lemma:approximation of supremum}.

\ignore{The generalization error of $H$ is shown to be bounded from below by
the VC-dimension of $H$ \citep{anthony1999neural}, which in turn is
lower bounded by the Rademacher complexity of $H$ \citep{mohri2012}.
Therefore, our lower bound on the Rademacher complexity of $H$ will
directly imply a lower bound on the generalization error.
}

\begin{theorem}
\label{theorem:lower_bound}
For any choice of $m$, $r$ there exists
samples $S$ and $U$, a setting of the regularization parameter
$\kyfan$, as well as a choice of base kernels $K_1,\ldots,K_p$
that guarantees
\begin{equation*}
 \hat{\mathfrak{R}}_{S}(H)  \geq \sqrt{\frac{\kyfan}{2m}} \,.
\end{equation*}
\end{theorem}
\begin{proof}
First we let $S$ and $U$ be any two samples, both of size
$m$, such that the $U$ is simply an unlabeled version of $S$.  Now, assume $p=1$ and the sample kernel matrix $\ovK_1$ of kernel $K_1$ has exactly $r$ distinct non-zero simple eigenvalues. Finally, select $\kyfan$ such that $ \kyfan / \lambda_1(\ovK_1) \leq 1$.

As calculated in Section \ref{sec:generalization}, $\sup_{\| \bw \|
\leq 1} \sum_{n = 1}^m \sigma_n h(x_{n})=\| \Pi_{U}\sum_{n =
1}^m \sigma_n \bPhi(x_{n})\|$ and in this particular scenario
$\|C_{U}\|_{(r)}=\|C_{S}\|_{(r)}$, thus the empirical Rademacher complexity simplifies to $
 \hat{\mathfrak{R}}_{S}(H)  = \frac{1}{m} \E_{\ssigma} \big[
   \sup_{\| C_{S} \|_{(r)} \leq \kyfan
}\| \Pi_{S} \sum_{n = 1}^m \sigma_n \bPhi(x_{n}) \| \big]$, where the projection can be written directly in terms of the
sample $S$ and the $L_1$ constraint on $\Mu$ is not needed since it is
satisfied by the Ky Fan $r$-norm constraint when $\kyfan \leq
\lambda_1(\ovK_1)$.

Now, let $u_1,\cdots,u_{r}$ denote the top $r$ eigenfunctions of
$C_{S}$, then following the steps from Lemma~\ref{lemma:inner_product_form} we can express the norm of projection as
\ignore{
\begin{equation}
\| \Pi_{S} \sum_{n = 1}^m \sigma_n \bPhi(x_{n})
\|=\sqrt{\sum_{i=1}^{r}\iprod{u_{i}}{\sum_{n = 1}^m \sigma_n
\bPhi(x_{n})}^{2}_{\mathbb{H}_{\mu K}}} \,.
\end{equation}
By \cite[Equation 18]{blanchard2008finite} $u_{i}$ takes the form
$
  u_{i}=\frac{1}{\sqrt{\lambda_{i} m }}\big(\sum_{j=1}^{m}K(\cdot,x_{j})[\bv_{i}]_{j}\big)
  \,,
$
where $(\lambda_i,\bv_{i})$ and $(\lambda_i,u_{i})$ are the
eigenpairs of $\ovK$ and $C_{S}$ respectively.
Recall that $\bPhi(x_{n}) = K(\cdot,x_{n})$ and evaluate the inner
product $\iprod{u_{i}}{\sum_{n = 1}^m \sigma_n
\bPhi(x_{n})}_{\mathbb{H}_{K}}$ as follows
\begin{multline*}
\iprod{u_{i}}{\sum_{n = 1}^m \sigma_n K(\cdot , x_{n})}
 = \sum_{n = 1}^m \sigma_n \iprod{u_{i}}{ K(\cdot , x_{n})} 
 = \sum_{n = 1}^m \sigma_n u_{i}(x_{n}) \\
 = \sum_{n = 1}^m \sigma_n \frac{1}{\sqrt{\lambda_{i} m }}\Bigg(\sum_{j=1}^{m}K(x_{n},x_{j})[\bv_{i}]_{j}\Bigg) 
 = \sqrt{m} \sum_{n = 1}^m \sigma_n \frac{1}{\sqrt{\lambda_{i}
}}\Bigg(\sum_{j=1}^{m}\frac{1}{m}K(x_{n},x_{j})[\bv_{i}]_{j}\Bigg) \,.
\end{multline*}
Observe that $[\ovK]_{n,j}=\frac{1}{m}K(x_{n},x_{j})$, thus
$\sum_{j=1}^{m}\frac{1}{m}K(x_{n},x_{j})[\bv_{i}]_{j}=\lambda_{i}[\bv_{i}]_{n}$,
which shows
\begin{equation}
\iprod{u_{i}}{\sum_{n = 1}^m \sigma_n K(\cdot , x_{n})}=\sqrt{m}
\sqrt{\lambda_{i}} \sum_{n = 1}^m \sigma_n [\bv_{i}]_{n}=\sqrt{m}
\sqrt{\lambda_{i}} \ssigma^{\top}\bv_{i} \,.
\end{equation}
Using the above expression for the inner product,  we find the norm
of projection to be
}
\begin{align}
\| \Pi_{S} \sum_{n = 1}^m \sigma_n \bPhi(x_{n}) \| & =
  \sqrt{m \sum_{i=1}^r  \lambda_{i}(\ssigma^{\top}\bv_{i})^{2}} \\
 &  = \sqrt{m \sum_{j=1}^r  \mu_1
\lambda_{j}(\ovK_1)(\ssigma^{\top}\bv_{1,j})^{2}} \,,
\end{align}
where $(\mu_1 \lambda_{j}(\ovK_1),\bv_{1,j})$ is the eigenpair of
normalized sample kernel matrix $\mu_1 \ovK_1 = \ovK$. 
The expression is furthermore simplified by introducing the vectors
$\bu_{\mu}$ with entries $\mu_1\lambda_{j}(\ovK_1)$ and
$\bu_{\sigma}$ with entries $(\bv_{1,j}^{\top}\ssigma)^{2}$.  Note
that here, unlike in the general statement of
Lemma~\ref{lemma:inner_product_form}, the choice of $r$ entries that
appear in $\bu_\mu$ and $\bu_\sigma$ are not effected by the value of
$\Mu$, since there are in fact only $r$ non-zero eigenvalues total by
construction (i.e. there is one base kernel of rank $r$).  The choice
of $\Mu$, however, still affects the scale of the $r$ eigenvalues.

By the monotonicity of the square-root function and using the
definition of $\bu_\Mu$ as well as the dual norm we have
\begin{align}
  \sup_{\substack{\| C_{S} \|_{(r)} \leq \kyfan}} \sqrt{\bu_{\mu} \cdot
\bu_{\ssigma}}
 & = \sqrt{ \underset{\| \bu_{\Mu} \|_1 \leq \kyfan}{\sup}
\bu_{\mu} \cdot \bu_{\ssigma}} \\
& = \sqrt{\kyfan \Vert \bu_{\sigma} \Vert_{\infty}} \,.
\end{align} 
Thus, the Rademacher complexity is reduced to
\begin{align}
\begin{split}
\label{eq:reduced_lower}
\hat{\mathfrak{R}}_{S}(H)  
 & = \sqrt{\frac{\kyfan}{m}} \E_{\ssigma}
  \Big[ \sqrt{\max_{j \in [1,r]} (\bv^{\top}_{1,j} \ssigma)^2} \Big] \\
 & = \sqrt{\frac{\kyfan}{m}} \E_{\ssigma}
\Big[ \max_{j \in [1,r]}|\bv^{\top}_{1,j} \ssigma |\Big] \,.
\end{split}
\end{align}
Finally, we use Jensen's
inequality and Khintchine's inequality to show
\begin{align}
\E_{\ssigma}
\Big[ \max_{j \in [1,r]}|\bv^{\top}_{1,j} \ssigma  |\Big]
  & \geq \max_{j \in [1,r]} \E_{\ssigma}
\Big[ |\bv^{\top}_{1,j} \ssigma  |\Big] \\
  & \geq \max_{j \in [1,r]} 2^{-1/2} \| \bv_{1,j} \| = 2^{-1/2},
\end{align}
where the tight constant $2^{-1/2}$ used in Khintchine's inequality can
be found in Chapter II of \citep{nazarov2000}. Plugging this constant
back into equation ~\eqref{eq:reduced_lower} completes the theorem.
\qed
\end{proof}
The lower bound demonstrates the effect of the regularization
parameter $\kyfan$ as well as the tightness of the upper bound in
terms of $m$. 
\ignore{
We note the lower-bound does not demonstrate the
dependence on the number of base kernel $p$. On the other hand, the
$p$ appears only logarithmically in the upper bound and thus its
effect is relatively weak.
}

While Theorem~\ref{theorem:lower_bound} has shown the necessity of
the Ky Fan $r$-norm constraint, we will now give a small example that
illustrates the difference of projections (for example as seen in
Lemma~\ref{lemma:approximation of supremum}) must necessary depend on
the eigengap quantity.  This in turn motivates the regularization
$\sum_{k=1}^{p}\frac{1}{\mu_k} \leq \oneovermu$ which ensures that
eigengap is not arbitrarily small, since otherwise if $\Mu$ goes to
zero, then $\eigengap$ also goes to zero.  The fact that the eigengap
is essential for the concentration of projections has been known in
the matrix perturbation theory literature \citep{stewart1990matrix}.
The following proposition gives an example which shows that the
dependence on the eigengap is tight.
\begin{proposition}
There exists operators $A$ and $B$ such that
\begin{equation*}
\|P_{r}(A)-P_{r}(B)\|=\frac{2\|A-B\|}{\lambda_{r}(A)-\lambda_{r+1}(A)}
\,.
\end{equation*}
where $P_{r}(A)$ (resp. $P_{r}(B)$) is the orthogonal projection onto the top $r$ eigenspace of $A$ (resp. $B$).
\end{proposition}
\begin{proof}
Consider $r=1$ and $A$ and $B$ defined as follows
$A=\bigl(\begin{smallmatrix}
1 + \epsilon &0\\ 0 & 1
\end{smallmatrix} \bigr)$ and $B=\bigl(\begin{smallmatrix}
1 &0\\ 0& 1 + \epsilon
\end{smallmatrix} \bigr)$, thus $A-B=\bigl(\begin{smallmatrix}
\epsilon &0\\ 0&-\epsilon
\end{smallmatrix} \bigr)$, which implies that $\|A-B\|=\epsilon$. Also,
the 
eigengap is equal to $\lambda_1(A)-\lambda_{2}(A)=\epsilon$. Now
note that 
$P^{1}(A)$ is the projection onto $e_1=(1,0)^{\top}$
and $P^{1}(B)$ is the projection onto $e_2=(0,1)^{\top}$. Since $e_1$
and $e_2$ are orthogonal, this implies
$\|P^{1}(A)-P^{1}(B)\|=\|P^{1}(A)\|+\|P^{1}(B)\|=2$. On the
other hand, $\frac{2
\|A-B\|}{\lambda_1(A)-\lambda_{2}(A)}=\frac{2\epsilon}{\epsilon}=2$,
which complete the proof. \qed
\end{proof}

Assume that operator $C$ is defined together with a positive $\mu \in \Rset^{d}$ , then the stability of the r-eigenspace of
$\mu C$ is determined by $\frac{1}{\mu}
\frac{1}{\lambda_{r}(C)-\lambda_{r+1}(C)}$. When we have operators
$C_1,\dots,C_{p}$, with identical spectra such that
$\lambda_{r}(C_k)-\lambda_{r+1}(C_k)=\epsilon$ for each $k$ in
$[1,p]$, but each of them acting on mutually orthogonal subspaces, the
stability of the r-eigenspace of $C=\sum_{k=1}^{p}\mu_kC_k$ is
determined by
$\sum_{k=1}^{p}\frac{1}{\mu_k}\frac{1}{\lambda_{r}(C_k)-\lambda_{r+1}(C_k)}=\frac{1}{\epsilon}\sum_{k=1}^{p}\frac{1}{\mu_k}$.

This example clearly shows that $\sum_{k=1}^{p}\frac{1}{\mu_k}$
controls the stability of eigenspace when operators act on orthogonal
subspaces, which directly applies to linearly independent kernels.

\section{Discussion}
\label{sec:discussion}

Here, we briefly discuss the results presented.  Let us first
emphasize that our choice of the hypothesis class $H$
(Section~\ref{sec:Learning scenario}) is strongly justified a
posteriori by the learning guarantees we presented: both our upper and
lower bounds on the Rademacher complexity
(Sections~\ref{sec:generalization} and \ref{sec:lower bound}) suggest
controlling the quantities present in the definition of $H$. The
regularization parameters we provide can be tuned to directly bound
each of these crucial quantities and thereby limit the risk of
over fitting.

Second, we observe that the hypothesis class $H$ clearly motivates the
design of a single-stage coupled algorithm.  Such an algorithm would
be based on structural risk minimization (SRM) and seek to minimize
the empirical error over increasingly complex hypothesis sets, by
varying the parameters $\kyfan$ and $\oneovermu$, to trade-off
empirical error and model complexity.  Although the design and
evaluation of such an algorithm is beyond the scope of this paper, we
note that existing literature has empirically evaluated both learning
kernels with KPCA in an unsupervised (two-stage) fashion
\citep{zhuang2011,lin2011multiple} and applied dimensionality reduction
(single-stage training) with a fixed kernel function
\citep{fukumizu2004dimensionality}, \citep{gonen2014coupled}.  While these
existing algorithms do not directly consider the hypothesis class we
motivated, they can, in certain cases, still select a hypothesis
function that is found in our class. In particular, our learning
guarantees remain applicable to hypotheses chosen in a two-stage
manner, as long as the regularization constraints are
satisfied. Similarly the case $p = 1$ which corresponds to the
standard fixed-kernel supervised learning scenario is covered by our
analysis. We note that even in such cases, the bounds that we provide
would be the first to guarantee the generalization ability of the
algorithm via bounding the sample Rademacher complexity.

 \section{Algorithm}
\label{sec:algorithm}

In this section we obtain a computational expression for $h(x)$, where $x \in \Rset^{d}$ and $h \in H$. Moreover, we formulate a minimization problem for training $\Mu$ and $\bw$ as well as discuss ways to efficiently solve it by breaking into a series of convex sub-problems.

For the clarity of presentation we assume $\Pi_{U}$ is full rank and provide the expression for $h \in \Hset$ in the following lemma.
\begin{lemma}
\label{lemma:compute_h}
Let $x\in\Rset^{d}$, then for every $h\in H$ there exist real numbers $\{z_{k,j}\}_{k\in[1,p],j\in[1,m]}$ such that

\begin{equation}
h(x)=\sum_{k,j}\xi_{k,j}(\Mu)z_{k,j}\sqrt{\frac{\mu_k}{\lambda_{j}(\bK_k)}}\sum_{n=1}^{m}K_k(x,x_{n})[\bv_{k,j}]_{n}
\end{equation}

where 
\begin{equation}
\xi_{k,j}(\Mu)=
\begin{cases}
&1 \textit{ if } \mu_{k}\lambda_{j}(\bK_{k}) \textit{ in top r from  } \{\mu_{k}\lambda_{j}(\bK_{k})\}_{k,j}\\
&0 \textit{ otherwise }
\end{cases}
\end{equation}

with the following constraints
\begin{equation}
\sum_{k,j}z_{k,j}^{2} \leq 1
\end{equation}

\begin{equation}
\frac{1}{m}\sum_{k,j}\xi_{k,j}(\Mu)\mu_k\lambda_{j}(\bK_k) \leq \kyfan
\end{equation}

\begin{equation}
\sum_{k}\mu_{k} \leq 1
\end{equation}

\begin{equation}
\sum_{k}\frac{1}{\mu_{k}} \leq \nu
\end{equation}

\begin{equation}
\mu_{k} \geq 0
\end{equation}
\end{lemma}
\begin{proof}
First, to obtain the computational expression for $h(x_{n})=\iprod{\bw}{\Pi_U \bPhi(x)}_\Hset$, for the moment imagine that $\Pi_{U}$ is of full rank, then if $z_{1}(w),\cdots,z_{mp}(w)$ are the coordinates of $w$ in the span of eigenfunctions of $C_{U}$ and $z_{1}(\Pi_U \bPhi(x)),\cdots,z_{mp}(\Pi_U \bPhi(x))$ are the coordinates of $\Pi_U \bPhi(x)$ in that span, we will have $\iprod{\bw}{\Pi_U \bPhi(x)}_\Hset=\sum_{j=0}^{mp}z_{j}(w)z_{j}(\Pi_U \bPhi(x))$. Now, when we go back to the original scenario of rank $r$ projection $\Pi_{U}$, we introduce choice variables $\xi_{1},\cdots,\xi_{mp} \in {0,1}$, where $\xi_{i}=1$ if the $i$-th eigenfunction is chosen for projection and $\xi_{i}=0$ otherwise. Thus, the expression for $h(x)$ becomes
\begin{equation}
h(x)=\sum_{j=0}^{mp}\xi_{j}z_{j}(w)z_{j}(\Pi_U \bPhi(x))
\end{equation}

The assumption of linearly independent kernels allows us to break the sum above into $p$ components for each base kernel. We will do it by keeping two indices $k\in[1,p]$ and $j\in[1,m]$, which gives $h(x)=\sum_{k,j}\xi_{k,j}z_{k,j}(w)z_{k,j}(\Pi_U \bPhi(x))$. Observe that $z_{k,j}(\Pi_U \bPhi(x)) = \iprod{\bPhi(x)}{u_{k,j}}=u_{k,j}(x)$ and the steps in the proof of Lemma~\ref{lemma:inner_product_form} show that 
\begin{equation}
u_{k,j}(x)=\sqrt{\frac{\mu_k}{\lambda_{j}(\bK_k)}}\sum_{n=1}^{m}K_k(x,x_{n})[\bv_{k,j}]_{n}
\end{equation}
Moreover, varying $w$ in $\Hset$ for the purpose of our algorithm is equivalent to varying its coordinates $z_{k,j}(w)$, thus we will use variables $z_{k,j}$ instead of them with the constraint $\sum_{k,j}z_{k,j}^{2} \leq 1$. Given all this analysis, the computational expression for $h(x)$ becomes

The optimization variables are $z_{k,j}$ and $\mu_{k}$, while $\xi_{k,j}$ is determined by $\Mu$.
\qed
\end{proof}

When we define $\bz$ as vector with entries $z_{k,j}$ and have some convex loss function over the training sample $L(\Mu,\bz)$, the optimization problem is
\begin{equation}
\min_{\Mu,\bz}L(\Mu,\bz)
\end{equation}

subject to
\begin{equation}
\|\bz\| \leq 1
\end{equation}

\begin{equation}
\Mu \in \muset
\end{equation}

Note that the loss function includes the complicated term 
\begin{equation}
\sum_{k,j}\xi_{k,j}(\Mu)z_{k,j}\sqrt{\frac{\mu_k}{\lambda_{j}(\bK_k)}}\sum_{n=1}^{m}K_k(x,x_{n})[\bv_{k,j}]_{n}\,.
\end{equation}

For conciseness, define
\begin{equation}
c_{k,j}(x)=\sqrt{\frac{1}{\lambda_{j}(\bK_k)}}\sum_{n=1}^{m}K_k(x,x_{n})[\bv_{k,j}]_{n} \,,
\end{equation}
which allows us to write clearly $h(x,\bz,\Mu)=\sum_{k,j}\xi_{k,j}c_{k,j}(x)z_{k,j}\sqrt{\mu_{k}}$. We will make a substitution $w_{k,j}=z_{k,j}\sqrt{\mu_{k}}$. That substitution changes constraint $\|\bz\|\leq 1$ to $\sum_{k,j}\frac{w_{k,j}^{2}}{\mu_{k}}\leq 1$, which is a convex set in $\bw$ and $\Mu$ for $\mu_{k} > 0$.

This reduces the problem to the following constrained optimization problem
\begin{equation}
\min \frac{1}{n}\sum_{i=1}^{n}L\Bigg(\sum_{k,j}\xi_{k,j}c_{k,j}(x_{n})w_{k,j},y_{n}\Bigg)
\end{equation}subject to  
\begin{equation}
\sum_{k,j}\frac{w_{k,j}^{2}}{\mu_{k}}\leq 1
\end{equation}
and
\begin{equation}
\Mu \in\muset
\end{equation}

There are at least two possible ways to relax the problem. First, we can relax $\xi_{k,j}(\Mu)$ to no longer be a function of $\Mu$ and make $\xi_{k,j}\in\{0,1\}$ a discrete optimization variable instead with an additional constraint $\sum_{k,j}\xi_{k,j}=r$. The second option is even weaker: we can let $\xi_{k,j}$ be a continuous variable in the interval $[0,1]$ and keep the constraint $\sum_{k,j}\xi_{k,j}=r$. Investigation of those relaxation steps is a direction for further research.

\ignore{As can be seen, for a convex function $L$, provided that the choice variables $\xi_{k,j}$ are fixed, the problem above is convex minimization. First, consider the constraint set. The function $\sum_{k,j}\frac{w_{k,j}^{2}}{\mu_{k}}$ is convex as a sum of convex functions on $\Mu \in \muset$ and $w_{k,j} \in \Rset$. Note that if you take a separate function e.g. $f=\frac{w_{1,1}}{\mu_{1}}$, its Hessian will be 
\begin{equation} \left(
\begin{matrix} \frac{2}{\mu_{1}} & -\frac{2w_{1,1}}{\mu_{1}^{2}}\\
 -\frac{2w_{1,1}}{\mu_{1}^{2}} & \frac{2w_{1,1}^{2}}{\mu_{1}^{3}}
\end{matrix} \right)
\end{equation}
which is positive semidefinite on $\mu_{1}>0$, which yields an convex function. Thus, the constraint set that we use is an intersection of convex sets, so our constraint is convex. Next, $\sum_{k,j}\xi_{k,j}c_{k,j}(x_{n})w_{k,j}$ is a linear function in $\bw$, Thus for a convex function $L$ the objective $\frac{1}{n}\sum_{i=1}^{n}L\Bigg(\sum_{k,j}\xi_{k,j}c_{k,j}(x_{n})w_{k,j},y_{n}\Bigg)$ is convex.}

\section{Conclusion}

We presented a new analysis and generalization guarantees for the
scenario of coupled nonlinear dimensionality reduction with a learner kernel.  The hypothesis
class is designed with regularization constraints that are directly
motivated by the generalization guarantee, which we show lower bounds
for as well.  Our analysis invites the design of learning
algorithms for selecting hypotheses from this specifically tailored
class, either in a two-stage or a single-stage manner.

\bibliographystyle{apalike}

\appendix

\section{Proof of Lemma ~\ref{lemma:approximation of supremum}}
\begin{proof}
For the first part of the proof, let $C_{\mu_{k}}$ be the true covariance operator of kernel $\mu_{k}K_{k}$. Since for each $k$ both $C_{S,k}$ and $C_{U,k}$ approach $C_{\mu_{k}}$ with high probability, we will show a concentration bound on their difference that holds uniformly over $k \in[1,p]$ as well as $U$ and $S$. Using union bound for probabilities and Lemma~1 from \citep{zwald2006convergence} (equivalently Corollary~5 from \citep{shawe2003estimating}) with probability at least $1-\delta$ for all $k\in[1,p]$, 
\begin{equation}
\label{eq:concentration_of_covariance_operators}
\max\big[\|C_{\mu_{k}}-C_{S,k}\|_{\ov\Hset_{k}},\|C_{\mu_{k}}-C_{U,k}\|_{\ov\Hset_{k}}\big] \leq 2\mu_{k}M_{\delta} / \sqrt{m}\,,
\end{equation}
where $M_{\delta}=1+\sqrt{\frac{\log{(2p/\delta)}}{2}}$. We used $u > m$ to obtain the bound.

By triangle inequality and decomposition over orthogonal subspaces of $\ov\Hset=\bigoplus_{k=1}^{p}\ov\Hset_{k}$, the norm $\|\Pi_{S}-\Pi_{U}\|$ is bounded by $\sum_{k=1}^{p}\|P_{r}(C_{\mu_k})-P_{r}(C_{S,k})\|_{\ov\Hset_k}+\sum_{k=1}^{p}\|P_{r}(C_{\mu_k})-P_{r}(C_{U,k})\|_{\ov\Hset_k}$. By Theorem~3\footnote{
Note that the actual theorem we reference works
provided that $2\|C_{\mu_k}-C_{S,k}\|_{\ov\Hset_{k}} / \big(\lambda_{r}(C_{\mu_k})-\lambda_{r+1}(C_{\mu_k})\big) \leq 1/2 $, which stems from the convergence of power series used in the proof. Here, for the simplicity of presentation we multiply their bound by 4, which relaxes such a requirement} from \citep{zwald2006convergence} for each $k\in[1,p]$,
\begin{equation}
\label{eq:projection_uniform_bound}
\|P_{r}(C_{\mu_k})-P_{r}(C_{S,k})\|_{\ov\Hset_k} \leq \frac{8\|C_{\mu_k}-C_{S,k}\|_{\ov\Hset_{k}}}{\lambda_{r}(C_{\mu_k})-\lambda_{r+1}(C_{\mu_k})}\,,
\end{equation}
and a similar statement holds for projeciton with respect to sample $U$. 

We will use $\|\Mu\|_{1}\leq 1$ to make the bound in~\eqref{eq:concentration_of_covariance_operators} be $M_{\delta} / \sqrt{m}$ and we will decompose $\lambda_{r}(C_{\mu_k})-\lambda_{r+1}(C_{\mu_k})=\mu_{k}\big(\lambda_{r}(C_{k})-\lambda_{r+1}(C_{k})\big)\geq \mu_{k}\gapstar$, where $C_{k}$ is the true covariance operator of kernel $K_{k}$ and $\gapstar=\min_{k\in[1,p]}\big(\lambda_{r}(C_k)-\lambda_{r+1}(C_k)\big)$. Now $2M_{\delta} / \sqrt{m} \mu_{k}\gapstar$ is the uniform bound on the norm of projections in~\eqref{eq:projection_uniform_bound}. Summing up $\|P_{r}(C_{\mu_k})-P_{r}(C_{S,k})\|_{\ov\Hset_k}+\|P_{r}(C_{\mu_k})-P_{r}(C_{U,k})\|_{\ov\Hset_k}$ over $k$ and applying the uniform bound $2M_{\delta} / \sqrt{m}\mu_{k}\gapstar$, which holds for both samples $U$ and $S$, we conclude
\begin{equation}
\|\Pi_{S}-\Pi_{U}\| \leq \sum_{k=1}^{p}\frac{32M_{\delta}}{\mu_{k}\gapstar\sqrt{m}} \leq \frac{32M_{\delta}\oneovermu}{\gapstar\sqrt{m}}
\end{equation}
 
For the last part we use a simple series of inequalities to get
\begin{align}
| \| C_{U} \|_{(r)}-\| C_{S} \|_{(r)}| & \leq \sum_{i=1}^{r}|\lambda_{i}(C_{U})-\lambda_{i}(C_{S})| \\
& \leq \sqrt{r}\bigg(\sum_{i=1}^{r}|\lambda_{i}(C_{U})-\lambda_{i}(C_{S})|^{2}\bigg)^{1/2}\,,
\end{align}
 which is in turn bounded by $\sqrt{r}\|C_{U}-C_{S}\|$ using Hoffman-Wielandt inequality. Now, $\|C_{U}-C_{S}\|$ is simply bounded by $\sum_{k=1}^{p}\|C_{\mu_{k}}-C_{S,k}\|_{\ov\Hset_{k}}+\sum_{k=1}^{p}\|C_{\mu_{k}}-C_{U,k}\|_{\ov\Hset_{k}}$. If we apply the uniform bound from~\eqref{eq:concentration_of_covariance_operators} in the form $\mu_{k}M_{\delta} / \sqrt{m}$, we get that with probability at least $1-\delta$
\begin{equation}
| \| C_{U} \|_{(r)}-\| C_{S} \|_{(r)}| \leq \sum_{k=1}^{p}\frac{\sqrt{r}\mu_{k}4M_{\delta}}{\sqrt{m}} \leq 4M_{\delta}\,.
\end{equation}

Putting together the two main bounds, we have that:
\begin{equation}
\sup_{\Mu \in \muset }\| \Pi_{U}u\|
 \leq \sup_{ \Mu \in \nuset}\| \Pi_{U}u\| 
 \leq \sup_{ \Mu \in \nuset}  \Big(\|
\Pi_{S}u\| + \frac{32M_{\delta}\oneovermu }{\eigengap\sqrt{m}}
\| u \| \Big) \,,
\end{equation}
\qed
\end{proof}

\section{Lemma~\ref{lemma:inner_product_form} with proof}
\begin{lemma}
\label{lemma:inner_product_form}
Let $C_{S,k}$ be the sample covariance operator of kernel $\mu_kK_k$. Define $C_{S}$ (resp.
$C_{U}$) as $C_{S}=\sum_{k=1}^{p}C_{S,k}$ and $\Pi_{S}$ (resp.
$\Pi_{U}$) be the orthogonal projection onto the eigenspace of
$\lambda_{i}(C_{S})$ for $i \in [1,r]$. Let $\bv_{k,j}$ be the eigenvector corresponding to $\lambda_j(\ovK_k)$.
If $I_{\mu}$ is an indexing set
that contains pairs $(k,j)$ that correspond to largest $r$ eigenvalues
from the set $\{\mu_k\lambda_{j}(\ovK_k)\}_{k,j}$, then 
\begin{align}
\| \Pi_{S} \sum_{n = 1}^m \sigma_n \bPhi(x_{n}) \|^{2}
& = m\sum_{(k,j) \in I_\Mu} \mu_k \lambda_j(\ovK_k)
 (\bv_{k,j}^\top \ssigma)^2 \\
& = m\bu_{\mu} \cdot \bu_{\sigma} \,,
\end{align}
where $\bu_{\mu}$ is a vector with entries
$\mu_k\lambda_{j}(\ovK_k)$, and $\bu_{\sigma}$ is a vector with entries
$(\bv_{k,j}^{\top}\ssigma)^{2}$ such that $(k,j) \in I_{\mu}$.
\end{lemma}
\begin{proof}
By properties of orthogonal projection, $\| \Pi_{S} \sum_{n = 1}^m \sigma_n \bPhi(x_{n}) \|^{2}$ can be expressed as $\sum_{i=1}^{r}\| P_{\lambda_{i}}(C_{S}) \sum_{n = 1}^m \sigma_n \bPhi(x_{n}) \|^{2}$, where $P_{\lambda_{i}}(C_{S})$ is the orthogonal projection onto the eigenfunction that corresponds to $\lambda_{i}(C_{S})$. Recall from Section~\ref{sec:Learning scenario} that when base kernels are linearly
independent with respect to sample $S$, then $\ov
\Hset=\bigoplus_{k=1}^{p}\ov\Hset_k$, which implies that for every
$i \in [1,m]$ there exists a $(k,j)$ such that the eigenfunction of
$\lambda_{i}(C_{S})$ is equal to the eigenfunction of
$\mu_k\lambda_{j}(C_{S,k})$. Note that $j$ may not
be equal to $i$ since $\Mu$ influences the ordering of eigenvalues.
Thus, we define the indexing set $I_{\Mu}$ that
contains the $r$ pairs of indices $(k,j)$ that correspond to the $r$
largest eigenvalues in
$\{\mu_k\lambda_{j}(C_{S,k})\}_{k,j}$ for particular value of $\Mu$.

For any $f \in \ov \Hset$ we can express $\| P_{\lambda_{j}}(C_{S}) f
\|^{2}=\langle u_{j}, f
\rangle^{2}$, where $u_{j}$ is the eigenfunction corresponding to
$\lambda_{j}(C_{S})$. Since
$\ov \Hset=\bigoplus_{k=1}^{p}\ov \Hset_k$ we can express the norm of projection as $\| \Pi_{S} f \|^{2}=\sum_{(k,j) \in I_{\Mu}}\langle u_{k,j}, f \rangle^{2}$, where $u_{k,j}$ is an eigenfunction of $C_{S,k}$ with eigenvalue
$\lambda_{j}(C_{S,k})$. When $f=\sum_{n = 1}^m \sigma_n \bPhi(x_{n})$
observe that $u_{k,j}$ belongs to orthogonal component $\ov
\Hset_k$, therefore it suffices to take inner product in $\ov
\Hset_k$, which by the reproducing property is equal to $u_{k,j}(x_{n})$. By \cite[Equation 18]{blanchard2008finite} $u_{k,j}(x_{n})$ takes the form
\begin{equation}
  u=\sqrt{\frac{\mu_k}{\lambda_{j}(\ovK_k) m }}\sum_{i=1}^{m}K_k(\cdot,x_{i})[\bv_{k,j}]_{i} \,,
\end{equation} which results in the following series of equalities
\begin{align*}
& \langle u_{k,j}, \sum_{n = 1}^m \sigma_n \bPhi(x_{n}) \rangle =\\
& = \sum_{n = 1}^m \sigma_n\langle u_{k,j}, \bPhi(x_{n})\rangle \\
& =\sqrt{m} \sqrt{\frac{
\mu_k}{\lambda_{j}(\ovK_k)}} \sum_{n = 1}^m \sigma_n  \sum_{i=1}^{m}\big[\bv_{k,j} \big]_{i}\frac{1}{m}K_k(x_{n},x_{i})  \\
& =\sqrt{m} \sqrt{\frac{
\mu_k}{\lambda_{j}(\ovK_k)}} \lambda_{j}(\ovK_k) \sum_{n = 1}^m \sigma_n\big[\bv_{k,j} \big]_{n} \\
& =\sqrt{m}\sqrt{
\mu_k\lambda_{j}(\ovK_k)} \bv_{k,j}^{\top}\ssigma \,.
\end{align*}
Squaring the terms above and summing them up, we arrive at
$
\| \Pi_{S} \sum_{n = 1}^m \sigma_n \bPhi(x_{n}) \|^{2} =m
\sum_{(k,j) \in
I_{\Mu}}\mu_k\lambda_{j}(\ovK_k)(\bv_{k,j}^{\top}\ssigma)^{2}=m\bu_{\mu} \cdot \bu_{\sigma}
$
which complete the proof. \qed
\end{proof}

\end{document}

%% file: example.tikz
\begin{tikzpicture}[scale=0.5]
       \node (a) at (0,0) [xshift=-1cm,label={[label distance=-5.5cm]90:a) original data}]
         {
\begin{tikzpicture}
\draw[help lines,->,line width=1.5pt, color=black, opacity=0] (-3,-3) -- (3,-3) coordinate (xaxis);
\draw[help lines,->,line width=1.5pt,color=black,opacity=0] (-3,-2) -- (-3,2) coordinate (yaxis);

\draw[step=1cm,gray,very thin, dashed] (-3,-2) grid (3,2);

\draw[fill=blue,color=blue] (2,1) circle (0.2cm);
\draw[fill=blue,color=blue] (-2,1) circle (0.2cm);

\draw[fill=red,color=red] (2,-1) circle (0.2cm);
\draw[fill=red,color=red] (-2,-1) circle (0.2cm);

\draw[help lines, ->, line width=1.5pt,postaction=decorate, color=black] (0,0)--(0,1);
\draw[help lines,->, line width=1.5pt,postaction=decorate, color=black] (0,0)--(1,0);
\node[above] at (0.8,0.1) {$\bv_{1}$};
\node[left] at (-0.1,0.7) {$\bv_{2}$};

\foreach \x in {-2,-1,0,1,2}
    \draw (\x,-2.0) -- (\x,-2.1) node[anchor=north] {$\x$};
    
    \foreach \y in {-2,-1,0,1,2}
    \draw (-3.0,\y) -- (-3.1,\y) node[anchor=east] {$\y$};
\end{tikzpicture}
         };
        \node (b) at (a.east) [anchor=east,xshift=8.0cm,label={[label distance=-5.5cm]90:b) reduced data}]
         {
            \tdplotsetmaincoords{60}{125}
\begin{tikzpicture}
\draw[help lines,->,line width=1.5pt, color=black, opacity=0] (-3,-3) -- (3,-3) coordinate (xaxis);
\draw[help lines,->,line width=1.5pt,color=black,opacity=0] (-3,-2) -- (-3,2) coordinate (yaxis);

\draw[step=1cm,gray,very thin, dashed] (-3,-2) grid (3,2);

\draw[fill=blue,color=blue, opacity=0.7] (2.1,0) circle (0.2cm);
\draw[fill=blue,color=blue, opacity=0.7] (-2.1,0) circle (0.2cm);

\draw[fill=red,color=red] (1.9,0) circle (0.2cm);
\draw[fill=red,color=red] (-1.9,0) circle (0.2cm);

\draw[help lines, ->, line width=1.5pt,postaction=decorate, color=black] (0,0)--(0,1);
\draw[help lines,->, line width=1.5pt,postaction=decorate, color=black] (0,0)--(1,0);
\node[above] at (0.8,0.1) {$\bv_{1}$};
\node[left] at (-0.1,0.7) {$\bv_{2}$};

\foreach \x in {-2,-1,0,1,2}
    \draw (\x,-2.0) -- (\x,-2.1) node[anchor=north] {$\x$};
    
    \foreach \y in {-2,-1,0,1,2}
    \draw (-3.0,\y) -- (-3.1,\y) node[anchor=east] {$\y$};
\end{tikzpicture}
         };

      
\end{tikzpicture}

%% file: scenario.tikz
  \tikzstyle{startstop} = [rectangle, rounded corners, minimum width=3cm, minimum height=1cm,text centered, draw=black, fill=red!30]
  
  \tikzstyle{io} = [trapezium, trapezium left angle=70, trapezium right angle=110, minimum width=3cm, minimum height=1cm, text centered, draw=black, fill=blue!30]
  
    \tikzstyle{sample} = [rectangle, minimum width=3cm, minimum height=2cm, text centered, draw=green!60!black, line width=2pt]

\tikzstyle{process} = [rectangle, minimum width=3cm, minimum height=1cm, text centered, draw=black, fill=orange!30]

\tikzstyle{processBlue} = [rectangle, minimum width=3cm, minimum height=1cm, text centered, draw=blue, line width=2pt]

\tikzstyle{processRed} = [rectangle, minimum width=3cm, minimum height=1cm, text centered, draw=red,line width=2pt]
\tikzstyle{processYellow} = [rectangle, minimum width=3cm, minimum height=1cm, text centered, draw=orange, line width=2pt]
\tikzstyle{processGreen} = [rectangle, minimum width=3cm, minimum height=1cm, text centered, draw=green, line width=2pt]
 \tikzstyle{startstopRed} = [rectangle, rounded corners, minimum width=3cm, minimum height=1cm,text centered, draw=red, line width=2pt]

\tikzstyle{decision} = [diamond, minimum width=3cm, minimum height=1cm, text centered, draw=black, fill=green!30]

\tikzstyle{samplepoint} = [circle, radius=1cm, draw=black]

\tikzstyle{arrow} = [thick,->,>=stealth]

\begin{tikzpicture}[node distance=2cm]
\node (input) [processBlue] {\begin{tabular}{c} Base kernels $K_{1}, \ldots, K_{p}$ \end{tabular}    };
\node (estimate1) [processYellow, below of=input] {$C_{S'}$};
\node (S') [sample, left of=estimate1, xshift=-3cm, label={[label distance=0.2cm]90:Unlabeled sample $S'$}] {};
\draw [arrow] (input) -- (estimate1);
\node(pt1)[samplepoint,below of =S', yshift=1.5cm, xshift=-1cm]{};
\node(pt1)[samplepoint,below of =S', yshift=2.0cm, xshift=0cm]{};
\node(pt1)[samplepoint,below of =S', yshift=2.3cm, xshift=1cm]{};
\node(pt1)[samplepoint,below of =S', yshift=2.7cm, xshift=-1cm]{};
\node (estimate2) [processBlue, below of=estimate1] {Compute $P^{r}(C_{S'})$};
\draw [arrow] (estimate1) -- (estimate2);


\node (train) [processRed, below of=estimate2] { Learn $\Mu$ and $\bw$ };

\node (S) [sample, right of=train, xshift=-7cm, label={[label distance=0.2cm]90:Labeled sample $S$}] {};
\node(pt1)[samplepoint,below of =S, yshift=1.5cm, xshift=-1cm, fill=red]{};
\node(pt1)[samplepoint,below of =S, yshift=2.0cm, xshift=0cm, fill=blue]{};
\node(pt1)[samplepoint,below of =S, yshift=2.3cm, xshift=1cm, fill=red]{};
\node(pt1)[samplepoint,below of =S, yshift=2.7cm, xshift=-1cm, fill=blue]{};
\draw [arrow] (S) -- (train);
\draw [arrow] (estimate2) -- (train);


\draw [arrow] (S') -- (estimate1);
\end{tikzpicture}

%% file: coupled.tikz
\begin{tikzpicture}
       \node (a) at (0,0) [xshift=-1cm,label={[label distance=-4.3cm]90:\Large(a)}]
         {
\begin{tikzpicture}

	\foreach \x in {0,40,...,360}{
	\filldraw[fill=blue] ({cos(\x)},{sin(\x)}) circle (0.1cm);
}

\filldraw[fill=red] (0.2,0.2) circle (0.1cm);
\filldraw[fill=red] (-0.2,0.2) circle (0.1cm);
\filldraw[fill=red] (0.2,-0.2) circle (0.1cm);
\filldraw[fill=red] (-0.2,-0.2) circle (0.1cm);
\end{tikzpicture}
         };
        \node (b) at (a.east) [anchor=east,xshift=6.5cm,label={[label distance=-5.5cm]90:\Large(b)}]
         {
            \tdplotsetmaincoords{60}{125}
\begin{tikzpicture}[
		tdplot_main_coords,
		grid/.style={very thin,gray},
		axis/.style={->,blue,thick},
		cube/.style={dashed,left color=blue,right color=red, opacity=0.9},
		cube hidden/.style={very thick,dashed}]
	\foreach \x in {-0.5,0,...,2.5}
		\foreach \y in {-0.5,0,...,2.5}
		{
			\draw[grid] (\x,-0.5) -- (\x,2.5);
			\draw[grid] (-0.5,\y) -- (2.5,\y);
		}

	\draw[axis] (0,0,0) -- (3,0,0) node[anchor=west]{$x$};
	\draw[axis] (0,0,0) -- (0,3,0) node[anchor=west]{$y$};
	\draw[axis] (0,0,0) -- (0,0,3) node[anchor=west]{$z$};

	\draw[cube] (2,1,0) -- (0,1,0) -- (0,1,2) -- (2,1,2) -- cycle;
	\draw[very thick] (0,1,2) -- (2,1,2);
	\draw[very thick] (2,1,2) -- (2,1,0);
	
	
		\foreach \x in {0}{
	\shadedraw[ball color=blue, opacity=0.8] ({cos(\x)},-1,{sin(\x)}) circle (0.1cm);
}

		\foreach \x in {40}{
	\shadedraw[ball color=blue, opacity=0.8] ({cos(\x)},-1,{sin(\x)}) circle (0.1cm);
}

		\foreach \x in {80}{
	\shadedraw[ball color=blue, opacity=0.8] ({cos(\x)},-1,{sin(\x)}) circle (0.1cm);
}

		\foreach \x in {120}{
	\shadedraw[ball color=blue, opacity=0.8] ({cos(\x)},-1,{sin(\x)}) circle (0.1cm);
}

		\foreach \x in {160}{
	\shadedraw[ball color=blue, opacity=0.8] ({cos(\x)},-1,{sin(\x)}) circle (0.1cm);
}

		\foreach \x in {200}{
	\shadedraw[ball color=blue, opacity=0.5] ({cos(\x)},-1,{sin(\x)}) circle (0.1cm);
}

		\foreach \x in {240}{
	\shadedraw[ball color=blue, opacity=0.8] ({cos(\x)},-1,{sin(\x)}) circle (0.1cm);
}

		\foreach \x in {280}{
	\shadedraw[ball color=blue, opacity=0.8] ({cos(\x)},-1,{sin(\x)}) circle (0.1cm);
}

		\foreach \x in {320}{
	\shadedraw[ball color=blue, opacity=0.8] ({cos(\x)},-1,{sin(\x)}) circle (0.1cm);
}

	\shadedraw[ball color=red] (0.2,1,0.2) circle (0.1cm);
\shadedraw[ball color=red] (-0.2,1,0.2) circle (0.1cm);
\shadedraw[ball color=red] (0.2,1,-0.2) circle (0.1cm);
\shadedraw[ball color=red] (-0.2,1,-0.2) circle (0.1cm);

\end{tikzpicture}
         };
   
    \node (c) at (b.east) [anchor=east,xshift=6cm,label={[label distance=-5.5cm]90:\Large(c)}]
         {
            \tdplotsetmaincoords{60}{125}
\begin{tikzpicture}[
		tdplot_main_coords,
		grid/.style={very thin,gray},
		axis/.style={->,blue,thick},
		cube/.style={dashed,left color=blue,right color=red, opacity=0.9},
		cube hidden/.style={very thick,dashed}]
	\foreach \x in {-0.5,0,...,2.5}
		\foreach \y in {-0.5,0,...,2.5}
		{
			\draw[grid] (\x,-0.5) -- (\x,2.5);
			\draw[grid] (-0.5,\y) -- (2.5,\y);
		}

	\draw[axis] (0,0,0) -- (3,0,0) node[anchor=west]{$x$};
	\draw[axis] (0,0,0) -- (0,3,0) node[anchor=west]{$y$};
	\draw[axis] (0,0,0) -- (0,0,3) node[anchor=west]{$z$};



	
	
	\draw[very thick] (3,1,0) -- (-1,1,0);

  \begin{scope}[canvas is yx plane at z=0]
     \filldraw[fill=red] (1.5,0.4) circle (0.1cm);
     \filldraw[fill=red] (1.5,0.8) circle (0.1cm);
     
     \filldraw[fill=blue] (0.5,0.5) circle (0.1cm);
     \filldraw[fill=blue] (0.5,0.9) circle (0.1cm);
     \filldraw[fill=blue] (0.5,1.5) circle (0.1cm);
     \filldraw[fill=blue] (0.5,2) circle (0.1cm);
   \end{scope}

\end{tikzpicture}
         };

       \draw[-open triangle 90] (a) -- node[auto] {$\bPhi$} (b);
       \draw[-open triangle 90] (b) -- node[auto] {$\Pi\bPhi$} (c);
\end{tikzpicture}